\documentclass[aos,preprint]{imsart}

\RequirePackage[OT1]{fontenc}
\RequirePackage{amsthm,amsmath}
\RequirePackage[numbers]{natbib}
\usepackage{color}
\RequirePackage[colorlinks,citecolor=blue,urlcolor=blue]{hyperref}
\RequirePackage{hypernat}
\RequirePackage{soul}
\usepackage{amsfonts}
\usepackage{dcolumn}
\usepackage{textcomp}
\usepackage{longtable}
\usepackage{verbatim}
\usepackage{latexsym}
\usepackage{fancyvrb}
\usepackage{multirow}
\usepackage[linesnumbered,ruled,vlined]{algorithm2e}
\usepackage{url}

\usepackage{comment}

\usepackage{natbib}

\usepackage{float}
\usepackage[dvips]{graphicx}

\startlocaldefs
\numberwithin{equation}{section}
\theoremstyle{plain}

\newtheorem{theorem}{Theorem}
\newtheorem{lemma}{Lemma}

\endlocaldefs

\setattribute{journal}{name}{}
\begin{document}
\begin{frontmatter}

\title{Supplement to ``Reversible MCMC on Markov equivalence classes of sparse directed acyclic graphs"\protect\thanksref{T11}}
\runtitle{Reversible MCMC on   Markov equivalence classes of sparse DAGs}

 \thankstext{T11}{This work was supported partially by  NSFC (11101008,11101005,  71271211), 973 Program-2007CB814905, DPHEC-20110001120113, US NSF grants
DMS-0907632,CCF-0939370, DMS-0605165, DMS-1107000, SES-0835531 (CDI), ARO grant W911NF-11-1-0114, US NSF Science and Technology Center. }
\begin{aug}
\author{\fnms{Yangbo} \snm{He }\thanksref{mm1}\ead[label=e1]{heyb@math.pku.edu.cn}},
\author{\fnms{Jinzhu} \snm{ Jia}\thanksref{mm1,mm2}\ead[label=e2]{jzjia@math.pku.edu.cn}}
\and
\author{\fnms{Bin} \snm{Yu}\thanksref{mm2}\ead[label=e3]{binyu@stat.berkeley.edu}}


\runauthor{Y.B. He et al.}

\affiliation{School of Mathematical Sciences, Center of Statistical Science,  LMAM and LMEQF,   Peking University \thanksmark{mm1}\\ University of California, Berkeley  \thanksmark{mm2}}

\address{School of Mathematical Sciences,\\
         Peking University\\
         Beijing 100871, China \\
\printead{e1}\\
\phantom{E-mail:\ } \printead*{e2}\\}

\address{Department of Statistics \\UC,
       Berkeley,CA 94720\\
\printead{e3}}

\end{aug}


\begin{abstract}
This supplementary material includes three parts: some preliminary results,  four examples, an experiment, three new algorithms,  and  all proofs of the results in  the paper \cite{heaos}.
\end{abstract}

\end{frontmatter}

\tableofcontents

\def\refconstrainedset{9}
\def\refperfect{1}
\def\refmc{1}
\def\refmsetalg{1.1}
\def\refalgai{1.1.1}
\def\refalgaii{1.1.2}
\def\refalgaiii{1.1.3}
\def\refspro{1}
\def\refsprod{2}
\def\refSalgorithms{3.2}
\def\refSacceleration{3.2.2}
\def\eqrefperfectset{(3.3)}
\def\refvaliddef{5}
\def\eqrefconstraintedC{(3.2)}
\def\eqrefedgec{(3.1)}
\def\refsprog{1}
\def\refessential{2}
\def\refchickervalid{3}
\section{Two preliminary algorithms}\label{preliminary}

In this Section, we provide   algorithms introduced by   Dor and Tarsi \cite{dor1992simple}, and Chickering \cite{chickering1995transformational,chickering2002learning} respectively. These results are necessary to implement our proposed approach technically.

 Some definitions and   notation    are introduced first.  A  directed edge  of a DAG is \emph{compelled}   if it occurs in the corresponding completed PDAG, otherwise, the directed edge is \emph{reversible} and the corresponding parents are reversible parents.   Recall $N_x$ be   the set of all neighbors of $x$, $\Pi_x$ is   the set of all parent of $x$,   $N_{xy}=N_x\cap N_y$ and $\Omega_{x,y}=\Pi_x \bigcap N_y$ and the concept of ``strongly protected" is presented in Definition \refspro\, in the paper \cite{heaos}.

 Algorithm \ref{pdag2dag} generates    a  consistent extension of a PDAG  \cite{dor1992simple}.   Algorithm \ref{dag2cpdag} creates the corresponding completed PDAG of a DAG \cite{chickering1995transformational}. They are used to implement Chickering's approach.

\begin{algorithm}[h]
\caption{(Dor and Tarsi \cite{dor1992simple}) Generate a  consistent extension of a PDAG  }\label{pdag2dag}
\KwIn{ A  PDAG $\cal P$ that admits a consistent extension}
\KwOut{  A DAG $\cal D$ that is a consistent extension of $\cal P$.}
Let ${\cal D}:={\cal P}$;\\
\While{$\cal P$ is not empty}{
Select a vertex $x$   in $\cal P$ such that (1)
$x$  has no outgoing edges and (2) if $N_x$ is not empty, then every vertex in $N_x$ is adjacent to all vertices in $N_x\cup \Pi_x$.
\tcc{Dor and Tarsi \cite{dor1992simple}  show that a vertex $x$ with these properties   is guaranteed to exist if $\cal P$ admits a consistent extension. }
Let all undirected edges adjacent to $x$ be directed toward $x$ in ${\cal D}$ \\
Remove $x$ and all incident edges from     $\cal P$.
}
\Return{  $\cal D$}
\end{algorithm}

\begin{algorithm}[h]
\caption{(Chickering \cite{chickering1995transformational}) Create the completed PDAG of a DAG }\label{dag2cpdag}
\KwIn{ $\cal D$,  a DAG}
\KwOut{  The completed PDAG ${\cal C}$ of  DAG $\cal D$.}

Perform a topological sort on the vertices in $\cal D$ such that for any pair of vertices $x$ and $y$ in $\cal D$, $x$ must precede $y$ if $x$ is an ancestor of $y$\;

 Sort  the edges first in ascending order for incident vertices and then   in  descending order for outgoing vertices;
Label every edge in $\cal D$ as ¡°unknown¡±;

\While {there are edges labeled ¡°unknown¡± in $\cal D$ \label{whilehere}} {
 Let $x \to y$ be the lowest ordered edge that is labeled ¡°unknown¡±\\
 \For {every edge $w\to x$  labeled ¡°compelled¡±}
{\If{ $w$ is not a parent of y}
{  $x \to y$ and every edge incident into $y$ with ¡°compelled¡±\\
\ Goto \ref{whilehere}}
\Else {
  Label $w\to y$  with ¡°compelled¡±}}
  \If {there exists an edge $z\to  y$ such that $z = x$ and $z$ is not a parent of $x$}
{Label  $x \to y$ and all ¡°unknown¡± edges incident into $y$ with ¡°compelled¡±}
\Else{
  Label  $x \to y$ and all ¡°unknown¡± edges incident into $y$ with ¡°reversible¡±}
}
Let ${\cal C}={\cal D}$ and undirect all edges labeled "reversible" in $\cal C$.

\Return{completed PDAG $\cal C$}
\end{algorithm}

\section{Additional examples, experiment and algorithms}

This section include three parts: (1) some examples to illuminate the methods proposed in the paper \cite{heaos}, (2) an experiment about v-structures, and (3) three algorithms to test the conditions \textbf{iu}$_3$, \textbf{id}$_3$ and \textbf{dd}$_2$ in   Algorithm \refmsetalg\,  only based on $e_t$ in an efficient manner.

\subsection{Examples} Four examples are presented to illustrate operators, the generation of a resulting completed PDAG of an operator, the conditions of a perfect operator set, and the process of constructing a perfect operator set.

  \textbf{Example 1.} This example illustrates six operators on a completed PDAG $\cal C$ and their corresponding modified graphs.
 Figure \ref{operatorsix} displays  six operators: InsertU $x-z$, DeleteU $y-z$, InsertD   $x\to v$,  DeleteD $z\to v$, MakeV $z\to y\leftarrow u$, and RemoveV $z\to v\leftarrow u$.     After inserting an undirected edge $x-z$ into the initial graph ${\cal C}$, we  get a modified graph denoted as ${\cal P}_1$ in  Figure \ref{operatorsix}.
 By applying the other five operators to  ${\cal C} $ in Figure \ref{operatorsix} respectively, we can obtain other five corresponding modified graphs   ${\mathcal{P}}_2, {\mathcal{P}}_3, {\mathcal{P}}_4, {\mathcal{P}}_5$, and ${\mathcal{P}}_6$. Here the operator ``MakeV $z\to y\leftarrow u$" modifies $z-y-u$ to $z\to y\leftarrow u$ and the operator ``Remove $z\to v\leftarrow u$" modifies $z\to v\leftarrow u$ to $z- v- u$.  Notice that a modified graph might not be a PDAG though all modified graphs in this example are PDAGs.

  \begin{figure}[h]
 \centering
  \begin{minipage}[c]{0.25\linewidth}
{\unitlength=0.6mm \begin{picture}(50,60)(10,20)
  \thicklines
\put(30,55){\circle*{1.5}} \put(20,35){\circle*{1.5}}
\put(30,45){\circle*{1.5}} \put(30,25){\circle*{1.5}}
\put(40,35){\circle*{2}}

\put(30,57){\mbox{$ x $}} \put(17,31){\mbox{$ y $}}
 \put(30,47){\mbox{$z $}}
\put(27,21){\mbox{$ u $}}
\put(42,34){\mbox{$ v $}}

\put(30,55){\line(-1,-2){10}} \put(21,36){\line(1,1){8}}
\put(21,34){\line(1,-1){8}} \put(21,35){\vector(1,0){18}}
\put(31,44){\vector(1,-1){8}}
\put(31,26){\vector(1,1){8}}
\put(10,70){\parbox{3cm}{
\begin{center}
Initial Completed PDAG
\end{center}
 }}
\put(35,20){{\mbox{(${\cal C}$)}}}
\end{picture} }
     \end{minipage}
   \begin{minipage}[t]{0.24\linewidth}
   \centering
 { \unitlength=0.6mm \begin{picture}(50,60)(10,20)
  \thicklines
\put(30,55){\circle*{1.5}} \put(20,35){\circle*{1.5}}
\put(30,45){\circle*{1.5}} \put(30,25){\circle*{1.5}}
\put(40,35){\circle*{2}}

\put(30,57){\mbox{$ x $}} \put(17,31){\mbox{$ y $}}
 \put(32,46){\mbox{$z $}}
\put(27,21){\mbox{$ u $}}
\put(42,34){\mbox{$ v $}}

\put(30,55){\line(-1,-2){10}}\put(30,55){\line(0,-1){10}}  \put(21,36){\line(1,1){8}}
\put(21,34){\line(1,-1){8}} \put(21,35){\vector(1,0){18}}
\put(31,44){\vector(1,-1){8}}
\put(31,26){\vector(1,1){8}}
\put(10,65){\mbox{InsertU  $x-z$}}
\put(35,20){{\mbox{(${\cal P}_1$)}}}
\end{picture}

\begin{picture}(50,60)(10,20)
  \thicklines
\put(30,55){\circle*{1.5}} \put(20,35){\circle*{1.5}}
\put(30,45){\circle*{1.5}} \put(30,25){\circle*{1.5}}
\put(40,35){\circle*{2}}

\put(30,57){\mbox{$ x $}} \put(17,31){\mbox{$ y $}}
 \put(30,47){\mbox{$z $}}
\put(27,21){\mbox{$ u $}}
\put(42,34){\mbox{$ v $}}

\put(30,55){\line(-1,-2){10}}
\put(21,34){\line(1,-1){8}} \put(21,35){\vector(1,0){18}}
\put(31,44){\vector(1,-1){8}}
\put(31,26){\vector(1,1){8}}
\put(10,65){\mbox{DeleteU $y-z$ }}
\put(35,20){{\mbox{(${\cal P}_2$)}}}
\end{picture}\\}
     \end{minipage}
  \begin{minipage}[t]{0.24\linewidth}
   \centering
 {\unitlength=0.6mm \begin{picture}(50,60)(10,20)
  \thicklines
\put(30,55){\circle*{1.5}} \put(20,35){\circle*{1.5}}
\put(30,45){\circle*{1.5}} \put(30,25){\circle*{1.5}}
\put(40,35){\circle*{2}}

\put(30,57){\mbox{$ x $}} \put(17,31){\mbox{$ y $}}
 \put(30,47){\mbox{$z $}}
\put(27,21){\mbox{$ u $}}
\put(42,34){\mbox{$ v $}}

\put(30,55){\line(-1,-2){10}} \put(30,55){\vector(1,-2){9}}
\put(21,36){\line(1,1){8}}
\put(21,34){\line(1,-1){8}} \put(21,35){\vector(1,0){18}}
\put(31,44){\vector(1,-1){8}}
\put(31,26){\vector(1,1){8}}
\put(10,65){\mbox{InsertD $x\to v$ }}
\put(35,20){{\mbox{(${\cal P}_3$)}}}
\end{picture}

 \begin{picture}(50,60)(10,20)
  \thicklines
\put(30,55){\circle*{1.5}} \put(20,35){\circle*{1.5}}
\put(30,45){\circle*{1.5}} \put(30,25){\circle*{1.5}}
\put(40,35){\circle*{2}}

\put(30,57){\mbox{$ x $}} \put(17,31){\mbox{$ y $}}
 \put(30,47){\mbox{$z $}}
\put(27,21){\mbox{$ u $}}
\put(42,34){\mbox{$ v $}}

\put(30,55){\line(-1,-2){10}}
 \put(21,36){\line(1,1){8}}
\put(21,34){\line(1,-1){8}}
\put(21,35){\vector(1,0){18}}
\put(31,26){\vector(1,1){8}}
\put(10,65){\mbox{DeleteD $z\to v$ }}
\put(35,20){{\mbox{(${\cal P}_4$)}}}
\end{picture} }
     \end{minipage}
 \begin{minipage}[t]{0.24\linewidth}
   \centering
{\unitlength=0.6mm \begin{picture}(50,70)(10,20)
  \thicklines
\put(25,55){\circle*{1.5}} \put(20,35){\circle*{1.5}}
\put(30,45){\circle*{1.5}} \put(30,25){\circle*{1.5}}
\put(40,35){\circle*{2}}

\put(27,57){\mbox{$ x $}} \put(17,31){\mbox{$ y $}}
 \put(30,47){\mbox{$z $}}
\put(27,21){\mbox{$ u $}}
\put(42,34){\mbox{$ v $}}

\put(25,55){\line(-1,-4){5}}

 \put(29,44){\vector(-1,-1){8}}
\put(29,26){\vector(-1,1){8}}

\put(21,35){\vector(1,0){18}}
\put(31,44){\vector(1,-1){8}}
\put(31,26){\vector(1,1){8}}
\put(10,65){\mbox{MakeV  $z\to y\leftarrow u$  }}
\put(35,20){{\mbox{(${\cal P}_5$)}}}
\end{picture}

\vspace{-0.7cm}\begin{picture}(50,70)(10,20)
  \thicklines
\put(30,55){\circle*{1.5}} \put(20,35){\circle*{1.5}}
\put(30,45){\circle*{1.5}} \put(30,25){\circle*{1.5}}
\put(40,35){\circle*{2}}

\put(30,57){\mbox{$ x $}} \put(17,31){\mbox{$ y $}}
 \put(30,47){\mbox{$z $}}
\put(27,21){\mbox{$ u $}}
\put(42,34){\mbox{$ v $}}

\put(30,55){\line(-1,-2){10}} \put(21,36){\line(1,1){8}}
\put(21,34){\line(1,-1){8}} \put(21,35){\vector(1,0){18}}
\put(31,44){\line(1,-1){9}}
\put(31,26){\line(1,1){9}}
\put(10,65){\mbox{RemoveV $z\to v\leftarrow u$ }}
\put(35,20){{\mbox{(${\cal P}_6$)}}}
\end{picture} }
     \end{minipage}
\caption{Examples of six operators of PDAG $\cal C$. ${\cal P}_1$ to ${\cal P}_6$ are the modified graphs of six operators.}
\label{operatorsix}
\end{figure}
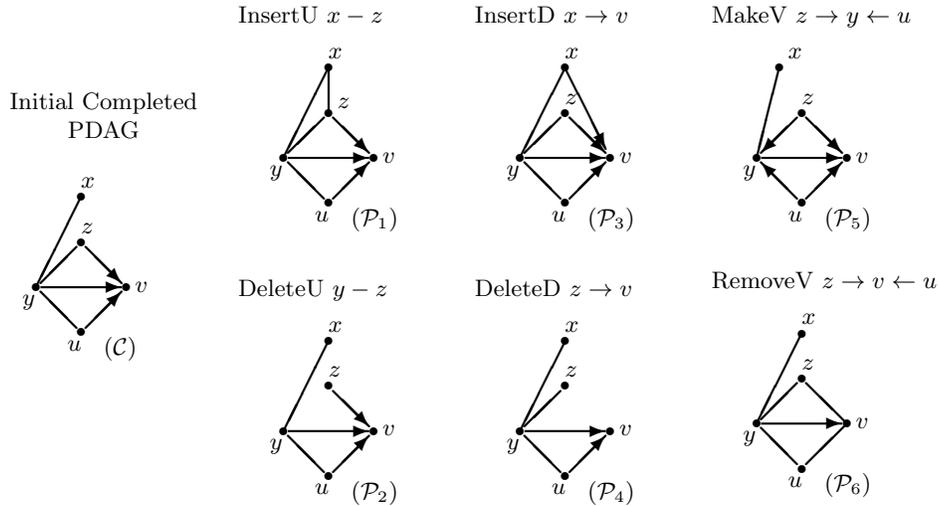

In the above example, we see that the modified graph of an operator, denoted by $\cal P$,  might   be a   PDAG, but might not be a completed PDAG. For example,  the modified graphs ${\cal P}_4$, and ${\cal P}_6$ in Figure \ref{operatorsix} are not  completed PDAGs  because  the directed edge $y\to v$ is not strongly protected.

\textbf{Example 2.} This example   illustrates  Chickering's approach to obtain the resulting completed PDAG of a valid operator from its modified graph.
Consider the initial completed PDAG $\cal C$ and the operator ``Remove  $z\to v\leftarrow u$" in Figure \ref{operatorsix}. We  illustrate  in Figure \ref{precedures} the steps of Chickering's approach that generates the resulting completed PDAG ${\cal C}_1$ by applying   ``Remove  $z\to v\leftarrow u$" to  ${\cal C}$.  The first step  (step 1)  extends the modified graph (a PDAG ${\cal P}_6$) to  a  consistent extension (${\cal D}_6$) via Algorithm \ref{pdag2dag}. The second step (step 2)   constructs the resulting    completed PDAG ${\cal C}_1$ of the operator  ``Remove  $z\to v\leftarrow u$" from the DAG ${\cal D}_6$ via Algorithm \ref{dag2cpdag}.

\begin{figure}[h]
 \centering
  \begin{minipage}[t]{0.20\linewidth}
{\unitlength=0.6mm \begin{picture}(50,60)(20,20)
  \thicklines
\put(30,55){\circle*{1.5}} \put(20,35){\circle*{1.5}}
\put(30,45){\circle*{1.5}} \put(30,25){\circle*{1.5}}
\put(40,35){\circle*{2}}

\put(30,57){\mbox{$ x $}} \put(17,31){\mbox{$ y $}}
 \put(30,47){\mbox{$z $}}
\put(27,21){\mbox{$ u $}}
\put(42,34){\mbox{$ v $}}

\put(30,55){\line(-1,-2){10}} \put(21,36){\line(1,1){8}}
\put(21,34){\line(1,-1){8}} \put(21,35){\vector(1,0){18}}
\put(31,44){\vector(1,-1){8}}
\put(31,26){\vector(1,1){8}}
\put(5,65){{\parbox{3cm}{\begin{center} Initial\\ completed PDAG \end{center}}}}
\put(35,20){{\mbox{(${\cal C}$)}}}
\end{picture} }
     \end{minipage}
\hspace{-.6cm}
\begin{minipage}[t]{0.05\linewidth}
{\unitlength=0.6mm \begin{picture}(5,80)(0,20)
  \thicklines
\put(0,35){{\mbox{$\Longrightarrow$}}}
\put(-20,43){{\parbox{3cm}{\begin{center}(apply \\operator) \end{center}}}}
\end{picture} }
\end{minipage}
\hspace{0.3cm}
 \begin{minipage}[t]{0.20\linewidth}
   \centering
{\unitlength=0.6mm \begin{picture}(50,60)(17,20)
  \thicklines
\put(30,55){\circle*{1.5}} \put(20,35){\circle*{1.5}}
\put(30,45){\circle*{1.5}} \put(30,25){\circle*{1.5}}
\put(40,35){\circle*{2}}
\put(5,65){{\parbox{3cm}{\begin{center}modified\\ graph \end{center}}}}
\put(30,57){\mbox{$ x $}} \put(17,31){\mbox{$ y $}}
 \put(30,47){\mbox{$z $}}
\put(27,21){\mbox{$ u $}}
\put(42,34){\mbox{$ v $}}

\put(30,55){\line(-1,-2){10}} \put(21,36){\line(1,1){8}}
\put(21,34){\line(1,-1){8}} \put(21,35){\vector(1,0){18}}
\put(31,44){\line(1,-1){9}}
\put(31,26){\line(1,1){9}}

\put(35,20){{\mbox{(${\cal P}_6$)}}}
\end{picture} }
     \end{minipage}
     \hspace{-.8cm}
     \begin{minipage}[t]{0.05\linewidth}
{\unitlength=0.6mm \begin{picture}(5,80)(0,20)
  \thicklines
\put(0,35){{\mbox{$\Longrightarrow$}}}
\put(-20,40){{\parbox{3cm}{\begin{center}     step 1 \end{center}}} }
\end{picture} }
\end{minipage}
 \begin{minipage}[t]{0.20\linewidth}
   \centering
{\unitlength=0.6mm \begin{picture}(50,60)(17,20)
  \thicklines
\put(30,55){\circle*{1.5}} \put(20,35){\circle*{1.5}}
\put(30,45){\circle*{1.5}} \put(30,25){\circle*{1.5}}
\put(40,35){\circle*{2}}
\put(5,65){ {\parbox{3cm}{\begin{center}consistent\\ extension\end{center}}}}
\put(30,57){\mbox{$ x $}} \put(17,31){\mbox{$ y $}}
 \put(30,47){\mbox{$z $}}
\put(27,21){\mbox{$ u $}}
\put(42,34){\mbox{$ v $}}

\put(20,35){\vector(1,2){9.5}} \put(21,36){\vector(1,1){8}}
\put(21,34){\vector(1,-1){8}} \put(21,35){\vector(1,0){18}}
\put(39,36){\vector(-1,1){8.5}}
\put(39,34){\vector(-1,-1){8.5}}

\put(35,20){{\mbox{(${\cal D}_6$)}}}
\end{picture} }
     \end{minipage}
     \hspace{-.8cm}\begin{minipage}[t]{0.05\linewidth}
{\unitlength=0.6mm \begin{picture}(5,80)(0,20)
  \thicklines
\put(0,35){{\mbox{$\Longrightarrow$}}}
\put(-20,40){{\parbox{3cm}{\begin{center} step 2\end{center}}}    }
\end{picture} }
\end{minipage}
 \begin{minipage}[t]{0.20\linewidth}
   \centering
{\unitlength=0.6mm \begin{picture}(50,60)(17,20)
  \thicklines
\put(30,55){\circle*{1.5}} \put(20,35){\circle*{1.5}}
\put(30,45){\circle*{1.5}} \put(30,25){\circle*{1.5}}
\put(40,35){\circle*{2}}

\put(30,57){\mbox{$ x $}} \put(17,31){\mbox{$ y $}}
 \put(30,47){\mbox{$z $}}
\put(27,21){\mbox{$ u $}}
\put(42,34){\mbox{$ v $}}

\put(30,55){\line(-1,-2){10}} \put(21,36){\line(1,1){8}}
\put(21,34){\line(1,-1){8}} \put(21,35){\line(1,0){19}}
\put(31,44){\line(1,-1){9}}
\put(31,26){\line(1,1){9}}
\put(5,65){{\parbox{3cm}{\begin{center}the resulting\\ completed PDAG\end{center}}}}
\put(35,20){{\mbox{(${\cal C}_1$)}}}
\end{picture} }
     \end{minipage}
\caption{Example for constructing the unique resulting completed PDAG of a valid operator.
 An operator ``Remove  $z\to v\leftarrow u $" in Figure \ref{operatorsix} is applied to the initial completed PDAG $\cal C$ and
  finally results in the resulting completed PDAG ${\cal C}_1$.
}
\label{precedures}
\end{figure}
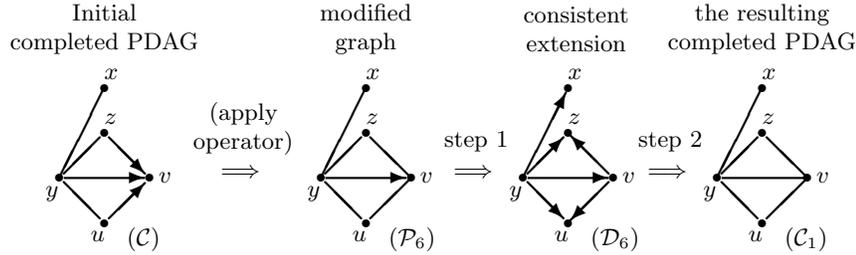

{\textbf{Example 3.}} This example illustrates that $\cal O$ in Equation \eqrefperfectset\,  will  not be reversible if  condition  \textbf{iu}$_3$ or  \textbf{dd}$_2$ is not contained in Definition \refconstrainedset.
  Consider the operator set $\cal O$ defined in   Equation \eqrefperfectset\, for ${\cal S}_5$ and   the   completed PDAG ${\cal C} \in {\cal S}_5$ in Figure \ref{antiexample1}. We have that operator InsertU $z-u$ and DeleteD $z\to v$ are valid. As shown in  Figure \ref{antiexample1}, InsertU $z-u$  transfers $\cal C$ to  the completed PDAG ${\cal C}_1$ and  DeleteD $z\to v$    transfers     $\cal C$ to the completed PDAG ${\cal C}_2$. However, deleting $z-u$ from ${\cal C}_1$
 will result in an undirected PDAG   distinct from     $\cal C$ and  InsertD $z\to v$ is
  not valid for ${\cal C}_2$. As a consequence, if  $\cal O$ contains InsertU $z-u$ and DeleteD $z\to v$, it will be not reversible. According to   Definition \refconstrainedset, these two operators do not appear in ${\cal O}_{\cal C}$ because they do not satisfy the conditions \textbf{iu}$_3$ and \textbf{dd}$_2$ respectively.

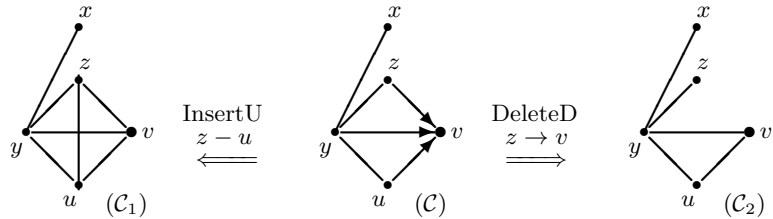
\begin{figure}[h]
 \centering
  \begin{minipage}[t]{0.28\linewidth}
{\unitlength=0.7mm \begin{picture}(50,50)(10,20)
  \thicklines
\put(30,55){\circle*{1.5}} \put(20,35){\circle*{1.5}}
\put(30,45){\circle*{1.5}} \put(30,25){\circle*{1.5}}
\put(40,35){\circle*{2}}

\put(30,57){\mbox{$ x $}} \put(17,31){\mbox{$ y $}}
 \put(30,47){\mbox{$z $}}
\put(27,21){\mbox{$ u $}}
\put(42,34){\mbox{$ v $}}

\put(30,55){\line(-1,-2){10}} \put(21,36){\line(1,1){8}}
\put(21,34){\line(1,-1){8}} \put(21,35){\line(1,0){18}}
\put(31,44){\line(1,-1){8}}
\put(31,26){\line(1,1){8}}
\put(30,46){\line(0,-1){22}}
\put(35,20){{\mbox{(${\cal C}_1$)}}}
\end{picture} }
     \end{minipage}
     \hspace{-.8cm}
     \begin{minipage}[t]{0.05\linewidth}
{\unitlength=0.6mm \begin{picture}(5,30)(0,0)
  \thicklines
\put(0,10){{\mbox{$\Longleftarrow \hspace{-0.1cm}=\hspace{-0.1cm}=$}}}
\put(-3,20){\mbox{InsertU}}
\put(0,15){\mbox{$z-u$}}
\end{picture} }
\end{minipage}
 \hspace{.3cm}
  \begin{minipage}[t]{0.28\linewidth}
{\unitlength=0.7mm \begin{picture}(50,50)(10,20)
  \thicklines
\put(30,55){\circle*{1.5}} \put(20,35){\circle*{1.5}}
\put(30,45){\circle*{1.5}} \put(30,25){\circle*{1.5}}
\put(40,35){\circle*{2}}

\put(30,57){\mbox{$ x $}} \put(17,31){\mbox{$ y $}}
 \put(30,47){\mbox{$z $}}
\put(27,21){\mbox{$ u $}}
\put(42,34){\mbox{$ v $}}

\put(30,55){\line(-1,-2){10}} \put(21,36){\line(1,1){8}}
\put(21,34){\line(1,-1){8}} \put(21,35){\vector(1,0){18}}
\put(31,44){\vector(1,-1){8}}
\put(31,26){\vector(1,1){8}}
\put(35,20){{\mbox{(${\cal C}$)}}}
\end{picture} }
     \end{minipage}
     \hspace{-.8cm}
     \begin{minipage}[t]{0.05\linewidth}
{\unitlength=0.6mm \begin{picture}(5,30)(0,0)
  \thicklines
\put(0,10){{\mbox{$=\hspace{-0.1cm}=\hspace{-0.1cm}\Longrightarrow $}}}
\put(-3,20){\mbox{DeleteD}}
\put(0,15){\mbox{$z\to v$}}
\end{picture} }
\end{minipage}
 \hspace{.3cm}
 \begin{minipage}[t]{0.28\linewidth}
{\unitlength=0.7mm \begin{picture}(50,50)(10,20)
  \thicklines
\put(30,55){\circle*{1.5}} \put(20,35){\circle*{1.5}}
\put(30,45){\circle*{1.5}} \put(30,25){\circle*{1.5}}
\put(40,35){\circle*{2}}

\put(30,57){\mbox{$ x $}} \put(17,31){\mbox{$ y $}}
 \put(30,47){\mbox{$z $}}
\put(27,21){\mbox{$ u $}}
\put(42,34){\mbox{$ v $}}

\put(30,55){\line(-1,-2){10}} \put(21,36){\line(1,1){8}}
\put(21,34){\line(1,-1){8}} \put(21,35){\line(1,0){18}}
\put(31,26){\line(1,1){8}}
\put(35,20){{\mbox{(${\cal C}_2$)}}}
\end{picture} }
     \end{minipage}

\caption{ Example: Two valid operators bring about irreversibility. It shows valid conditions are not sufficient for perfect operator set.}
\label{antiexample1}
\end{figure}

{\textbf{Example 4.}} This toy example is given  to show how to construct a concrete perfect set of operators following Definition \refconstrainedset\, in the paper \cite{heaos}.
Consider the completed PDAG ${\cal C}$ in Example 3. Here we introduce the procedure to determine $InsertU_{\cal C}$.  All possible operators of inserting an undirected edge to $\cal C$ include: ``InsertU $x-z$", ``InsertU $x-u$", ``InsertU $x-v$" and  ``InsertU $z-u$". The operator ``InsertU $x-v$" is not valid  according to Lemma \refchickervalid\, in the paper \cite{heaos}   since $\Pi(x)\neq \Pi(v)$. The operator ``InsertU $z-u$" is valid; however,  condition \textbf{iu}$_3$  does not hold. According to   Definition \refconstrainedset\, in the paper \cite{heaos}, we have   that   only ``InsertU $x-z$" and ``InsertU $x-u$" are in $InsertU_{\cal C}$. Thus $InsertU_{\cal C}=\{x-z,x-u\}$, where ``$x-z$" denotes ``InsertU $x-z$" in the set. Table \ref{constrainedsetE} lists the six   sets of  operators on  ${\cal C}$.

\begin{table}[h]\caption{ The six   sets of operators of $\cal C$. These operators are perfect. }\label{constrainedsetE}
  \begin{tabular}{|p{70pt}|p{145pt}|p{140pt}|}

\hline

\raisebox{-3.00ex}[0cm][0cm]{\parbox{5cm}
{\unitlength=0.55mm\begin{picture}(50,60)(10,20)
  \thicklines
\put(30,55){\circle*{1.5}} \put(20,35){\circle*{1.5}}
\put(30,45){\circle*{1.5}} \put(30,25){\circle*{1.5}}
\put(40,35){\circle*{2}}

\put(30,57){\mbox{$ x $}} \put(17,31){\mbox{$ y $}}
 \put(30,47){\mbox{$z $}}
\put(27,21){\mbox{$ u $}}
\put(42,34){\mbox{$ v $}}

\put(30,55){\line(-1,-2){10}} \put(21,36){\line(1,1){8}}
\put(21,34){\line(1,-1){8}} \put(21,35){\vector(1,0){18}}
\put(31,44){\vector(1,-1){8}}
\put(31,26){\vector(1,1){8}}
\put(35,20){{\mbox{(${\cal C}$)}}}
\end{picture}}
}
   &  $InsertU_{\cal C}= \{x-z,x-u\}$&  $  DeleteU_{\cal C}=\{x-y,y-z,y-u\}$
      \\[0.2cm] \cline{2-3} &&\\[-0.2cm]
    &  $InsertD_{\cal C}=\{x\to u\}$& $DeleteD_{\cal C}=\{y\to v\}$
     \\[0.2cm] \cline{2-3} &&\\[-0.2cm]
      &  $\begin{array}{rl} MakeV_{\cal C}=\{x-y-z,x-y-u,&\\z-y-u\}&\end{array}$&  $ RemoveV_{\cal C}=\{u\to v\leftarrow z\} $
     \\[0.2cm] \cline{2-3}
  \hline
\end{tabular}
\end{table}

\subsection{Experiment about v-structures}
Below, we present the experiment result in   Figure \ref{vstr} about the    numbers of v-structures  of completed PDAGs  in $\mathcal{S}_p^{rp}$.

 For $\mathcal{S}_p^{1.5p}$ in the main  window,  the medians of  the four distributions are  108, 220, 557 and 1110 for $p$=100, 200, 500, and 1000 respectively.  Figure \ref{vstr} shows that the   numbers of v-structures are much  less than $(p^2)$ for most completed PDAGs in $\mathcal{S}_p^{rp}$ when $r$ is set to   $1.2, 1.5$ or $3$. This result  is useful to analyze the time complexities of Algorithm
\refmc\,   and Algorithm \refmsetalg \, in Section \refSacceleration \,  of the paper \cite{heaos}.

\begin{figure}[h]
\centering
\includegraphics[scale=0.3]{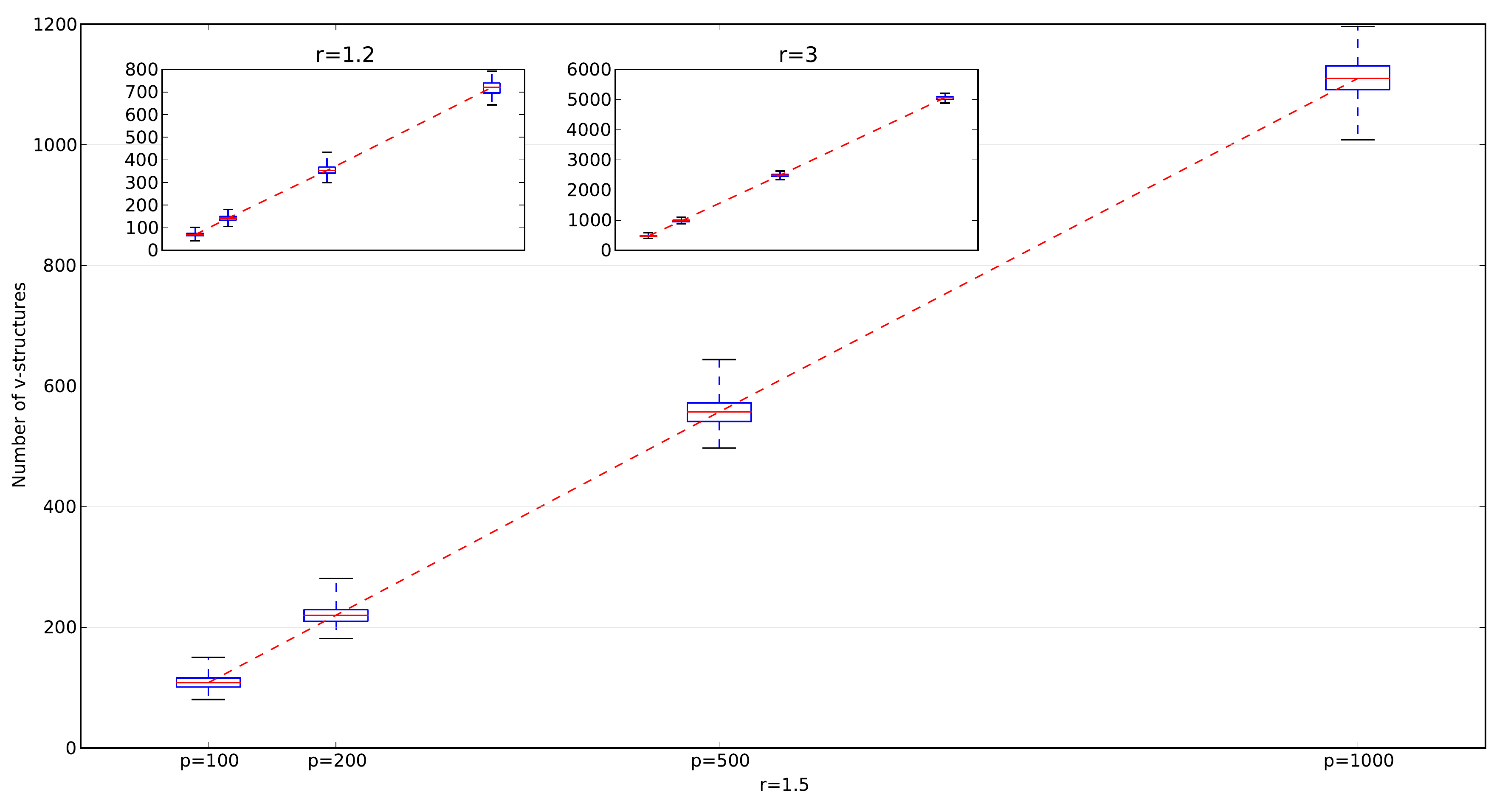}
\caption{The distributions of the numbers of v-structures  of completed PDAGs  in $\mathcal{S}_p^{rp}$. The red lines in the boxes indicate the medians.} \label{vstr}
\end{figure}

\subsection{Three Algorithms to check   \textbf{iu}$_3$, \textbf{id}$_3$ and \textbf{dd}$_2$ in   Algorithm \refmsetalg }\label{threealg}

  The conditions \textbf{iu}$_3$, \textbf{id}$_3$ and \textbf{dd}$_2$ in the fourth group depend  on both $e_t$ and the  resulting completed PDAGs of the  operators.  Intuitively,   checking  these three conditions requires that we obtain the corresponding resulting completed PDAGs.  We know that the time complexity of getting a resulting completed  PDAG  of $e_t$   is   $O(pn_{e_t})$ \cite{dor1992simple,chickering2002learning}, where $n_{e_t}$ is the number of edges in $e_t$.  To avoid generating resulting completed PDAG, we  provide three algorithms  to  check \textbf{iu}$_3$, \textbf{id}$_3$ and \textbf{dd}$_2$ respectively.

   In these three algorithms, we use the concept of strongly protected edges, defined in Definition \refsprod.   Let  $\Delta_v$ contain   all vertices adjacent to $v$.  To check whether a directed edge $v\to u$ is strongly protected or not in a graph  ${\cal G}$,  from Definition \refsprod, we need to   check    whether   one of
the   four configurations  in Figure \refspro \, occurs in $\cal G$.   This    can be implemented by local search in   $\Delta_v$ and  $\Delta_u$. We know that when a PDAG is sparse, in general, these sets   are small, so it is very efficient to check whether an edge is ``strongly protected".

We are now ready to  provide  Algorithm \ref{alga3}, Algorithm \ref{alga80}, and Algorithm \ref{alga10}   to check  \textbf{iu}$_3$, \textbf{id}$_3$ and \textbf{dd}$_2$  only based on  $e_t$, respectively.  In these three algorithms, we just need to check  whether  a few directed edges   are strongly protected or not in ${\cal P}_{t+1}$,  which has only one or a few edges different from $e_t$.  We prove in Theorem \ref{a3810theo} that these three algorithms are equivalent to   checking     conditions \textbf{iu}$_3$, \textbf{id}$_3$ and \textbf{dd}$_2$, respectively.

\vspace{0.2cm}
\begin{algorithm}[h]
\SetAlgoRefName{1.1.1}
\caption{Check the condition \textbf{iu}$_3$  in Definition \refconstrainedset}\label{alga3}
\SetKwInOut{Input}{input}\SetKwInOut{Output}{output}
\KwIn{a completed PDAG $ e_t $ and a valid operator   on    it: InsertU $x- y$.}
\KwOut{True or False}
\BlankLine
Insert $x- y$ to $e_t $, get the modified PDAG denoted as ${\cal P}_{t+1}$,\\
\For{ each common child $u$ of $x$ and $y$ in ${\cal P}_{t+1}$ } {
\If { either $x\to u$ or $y\to u$  is not strongly protected in ${\cal P}_{t+1}$ }{\Return {False}}}
\Return{True (\textbf{iu}$_3$ holds for InsertU $x- y$)}
\end{algorithm}

\begin{algorithm}[h]
\SetAlgoRefName{1.1.2}
\caption{Check  the  condition \textbf{id}$_3$ in Definition \refconstrainedset}
\label{alga80}
\SetKwInOut{Input}{input}\SetKwInOut{Output}{output}
\KwIn{a completed PDAG ${e_t} $ and  a valid operator: InsertD $x\to y$.}
\KwOut{True or False}
\BlankLine
Insert $x\to y$ to ${e_t} $, get a PDAG, denoted as ${\cal P}_{o}$,\\
\For{each undirected  edge   $u-y$ in ${\cal P}_{o}$, where $u$ is not adjacent to $x$}{ update ${\cal P}_{o}$ by  orienting $u-y$ to $y\to u$, }
\For{each edge   $v \to y$ in ${\cal P}_{o}$} {
\If {$v \to y$ is not strongly protected in ${\cal P}_{o}$ }{update ${\cal P}_{o}$ by changing  $v \to y$ to $v-y$,}}

Set ${\cal P}_{t+1}={\cal P}_{o}$

\For{each common child $u$ of $x$ and $y$ in ${\cal P}_{t+1}$}
 {
\If {  $y\to u$  is not strongly protected in ${\cal P}_{t+1}$ }{\Return {False}}}
\Return{True (\textbf{id}$_3$ holds for InsertD $x\to y$)}
\end{algorithm}

\begin{algorithm}[h]
\SetAlgoRefName{1.1.3}
\SetKwInOut{Input}{input}\SetKwInOut{Output}{output}
\caption{Check the condition \textbf{dd}$_2$  in Definition \refconstrainedset}
\label{alga10}

\KwIn{a completed PDAG ${e_t} $ and a valid operator DeleteD $x\to y$}
\KwOut{True or False}
\BlankLine
Delete $x\to y$ from ${e_t} $, get a PDAG, denoted as ${\cal P}_{t+1}$\;
\For{ each parent $v$ of  $y$ in ${\cal P}_{t+1}$} {
\If {  $v\to y$  is not strongly protected in ${\cal P}_{t+1}$ }{\Return {False}}}
\Return {True (\textbf{dd}$_2$ holds for DeleteD $x\to y$)}
\end{algorithm}

\begin{theorem}[Correctness of Algorithms \ref{alga3}, \ref{alga80} and \ref{alga10}]
\label{a3810theo}
 Let $e_t$ be a completed PDAG.   We have the following results.
  \begin{description}
    \item[(i) ]  Let InsertU $x-y$ be any valid  operator of $e_t$, then condition \textbf{iu}$_3$ holds for the operator InsertU $x- y$ if and only if the output of Algorithm \ref{alga3} is  True.
    \item[(ii) ]  Let InsertD $x\to y$ be any valid  operator of $e_t$, then condition \textbf{id}$_3$ holds for the operator InsertD $x\to y$  if and only if the output of Algorithm \ref{alga80} is  True.
    \item[(iii)] Let  DeleteD $x\to y$ be any valid  operator of $e_t$, then condition \textbf{dd}$_2$ holds for the operator  DeleteD $x\to y$  if and only if the output of Algorithm \ref{alga10} is  True.
  \end{description}
\end{theorem}

In Theorem \ref{a3810theo}, we show  that     an algorithm (Algorithm \ref{alga3}, Algorithm \ref{alga80}, or Algorithm \ref{alga10}) returns True for an operator if and only if the  corresponding condition (\textbf{iu}$_3$, \textbf{id}$_3$ or \textbf{dd}$_2$) holds for the operator. Theorem \ref{a3810theo} says that we do not have to examine the resulting completed PDAG to check conditions \textbf{iu}$_3$, \textbf{id}$_3$ and \textbf{dd}$_2$, which saves much computation time.

\section{Proofs}\label{proofs}

We will provide   a proof  of Theorem \ref{a3810theo}  in Subsection \ref{proof2}  below. Notice that we present Theorem \ref{a3810theo} in Subsection \ref{threealg} to show the correctness of Algorithm \refalgai, Algorithm \refalgaii \, and Algorithm \refalgaiii.

\subsection{Proof of Theorem \ref{a3810theo}  introduced in Subsection \ref{threealg}}\label{proof2}

To prove Theorem \ref{a3810theo}, we  need the following lemmas that have been introduced in the paper \cite{heaos}.

 \setcounter{lemma}{5}

 \begin{lemma}\label{re3}
 For any   operator $o\in {\cal O}_{{\cal C}}$  denoted by ``InsertD $x\to y$",  the  operator ``DeleteD $x\to y$" is  the reversible operator of $o$.
\end{lemma}

 \setcounter{lemma}{11}
  \begin{lemma}\label{cq}
 Let graph ${\cal C}$ be a completed PDAG, $\{w,v,u\}$ be three vertices that are adjacent each other in  ${\cal C}$. If there are two undirected edges in $\{w,v,u\}$, then the third edge is also undirected.
 \end{lemma}

 \setcounter{lemma}{13}
\begin{lemma}\label{redecs}

Let ${\cal C}$ be any completed PDAG, and let $\cal P$ denote the PDAG that results from adding a new edge  between $x$ and $y$. For any edge $v\to u$ in ${\cal C}$ that does not occur in the resulting completed PDAG extended from $\cal P$, there is a  directed path  of length zero or more from both $x$ and $y$ to $u$ in ${\cal C}$.
\end{lemma}

 \setcounter{lemma}{14}
 \begin{lemma}\label{re0str}

Let   $ {InsertU}_{\cal C}$ and $  {DeleteU}_{\cal C}$ be the  operator sets defined in Definition \refconstrainedset\, in the paper \cite{heaos}. For any $o$ in ${InsertU}_{\cal C}$ or in $ {DeleteU}_{\cal C}$,  where   ${\cal P}'$ is the modified graph of $o$ that is obtained by applying $o$ to  ${\cal C}$, we have that  ${\cal P}'$ is a completed PDAG.
\end{lemma}

 \setcounter{lemma}{16}
  \begin{lemma}\label{fdel}
If the graph ${\cal P}_1$ obtained by deleting  $a\to b$ from a  completed PDAG ${\cal C}$ can be extended to a new completed PDAG, ${\cal C}_1$, then we have that for any directed edge $x\to y$ in ${\cal C}$, if $y$ is not $b$ or a descendent of $b$, then $x\to y$ occurs in ${\cal C}_1$.
\end{lemma}

 \setcounter{lemma}{20}

There are three statements  in Theorem  \ref{a3810theo};   we prove them one by one below.

{\textbf{ Proof of  (i) of Theorem \ref{a3810theo}  }}

\textbf{(If)}

 Figure \refsprog\, shows the four cases that ensure that an edge is strongly protected.  We first show that for any edge $x\to u$ (or $y\to u$) , where $u$ is   a    common child of $x$ and $y$, if $x\to u$ is strongly protected in  ${\cal P}_{t+1}$   by configuration $(a)$, $(b)$, or $(d)$    in Figure \refsprog\, (replace $v\to u$ by  $x\to u$), it is also directed in ${e}_{t+1}$.

Case $(1)$, $(2)$ and $(3)$ in Figure \ref{spro1} show the sub-structures of  ${\cal P}_{t+1}$ in which $x\to u$ is protected by case $(a)$ , $(b)$ and $(d)$ in Figure \refsprog\, respectively, where ${\cal P}_{t+1}$ is the modified graph obtained by inserting $x-y$ into $e_t$.

 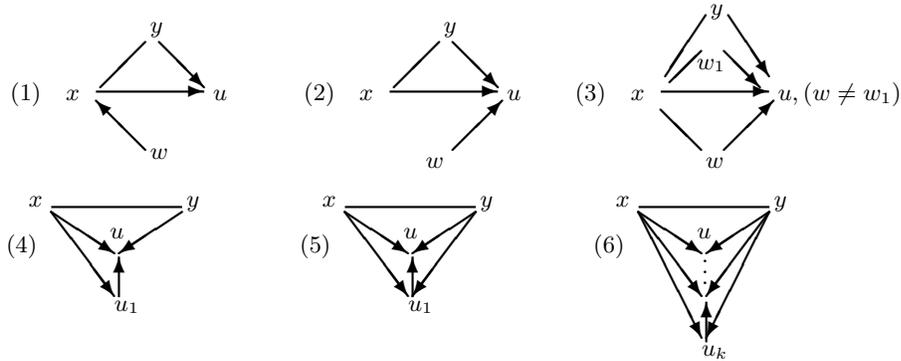
\begin{figure}[H]
 \centering
  \begin{minipage}[t]{0.3\linewidth}
   \centering
 { \unitlength=0.6mm \begin{picture}(50,40)
  \thicklines

\put(0,23){\mbox{$(1)\quad x $}}
\put(45,23){\mbox{$ u$}}
\put(31,38){\mbox{$y$}}
\put(31,10){\mbox{$w$}}

\put(33,36){\vector(1,-1){10}}
\put(30,36){\line(-1,-1){10}}

\put(19,25){\vector(1,0){24}}
\put(30,12){\vector(-1,1){11}}
\end{picture}\\ }
     \end{minipage}
 \begin{minipage}[t]{0.3\linewidth}
   \centering
{ \unitlength=0.6mm
\begin{picture}(50,40)
  \thicklines

\put(0,23){\mbox{$(2)\quad x $}}
\put(45,23){\mbox{$ u$}}
\put(27,8){\mbox{$w$}}
\put(31,38){\mbox{$y$}}

\put(33,36){\vector(1,-1){10}}
\put(30,36){\line(-1,-1){10}}

\put(19,25){\vector(1,0){24}}
\put(33,12){\vector(1,1){11}}

\end{picture}\\ }
  \end{minipage}
    \begin{minipage}[t]{0.3\linewidth}
   \centering
{ \unitlength=0.6mm \begin{picture}(60,40)(0,0)
  \thicklines
\put(0,23){\mbox{$(3)\quad x $}}
\put(45,23){\mbox{$ u, (w\neq w_1)$}}
\put(29,8){\mbox{$w$}}
\put(27,30){\mbox{$ w_1$}}
\put(30,42){\mbox{$ y$}}

\put(29,42){\line(-2,-3){9}}
\put(34,42){\vector(2,-3){9}}

\put(19,25){\vector(1,0){24}}
\put(33,12){\vector(1,1){11}}
\put(28,12){\line(-1,1){9}}

\put(33,34){\vector(1,-1){8}}
\put(28,34){\line(-1,-1){7}}
\end{picture}\\ }
  \end{minipage}

  \begin{minipage}[t]{0.3\linewidth}
   \centering
 { \unitlength=0.6mm \begin{picture}(50,40)
  \thicklines

\put(5,40){\mbox{$x $}}
\put(0,30){\mbox{$(4)$}}
\put(40,40){\mbox{$y$}}
\put(23,33){\mbox{$u$}}
\put(24,17){\mbox{$u_1$}}

\put(10,40){\line(1,0){29}}

\put(10,39){\vector(3,-2){14}}
\put(39,39){\vector(-3,-2){14}}

\put(10,39){\vector(3,-4){14}}
\put(25,20){\vector(0,1){9}}
\end{picture}\\ }
     \end{minipage}
 \begin{minipage}[t]{0.3\linewidth}
   \centering
 { \unitlength=0.6mm \begin{picture}(50,40)
  \thicklines

\put(5,40){\mbox{$x $}}
\put(0,30){\mbox{$(5)$}}
\put(40,40){\mbox{$y$}}
\put(23,33){\mbox{$u$}}
\put(24,17){\mbox{$u_1$}}

\put(10,40){\line(1,0){29}}

\put(10,39){\vector(3,-2){14}}
\put(39,39){\vector(-3,-2){14}}

\put(10,39){\vector(3,-4){14}}
\put(39,39){\vector(-3,-4){14}}
\put(25,20){\vector(0,1){9}}
\end{picture}\\ }
     \end{minipage}
 \begin{minipage}[t]{0.3\linewidth}
   \centering
 { \unitlength=0.6mm \begin{picture}(50,40)
  \thicklines

\put(5,40){\mbox{$x $}}
\put(0,30){\mbox{$(6)$}}
\put(40,40){\mbox{$y$}}
\put(23,33){\mbox{$u$}}
\put(24,7){\mbox{$u_k$}}

\put(10,40){\line(1,0){29}}

\put(10,39){\vector(3,-2){14}}
\put(39,39){\vector(-3,-2){14}}

\put(10,39){\vector(3,-4){14}}
\put(39,39){\vector(-3,-4){14}}
\put(24,23){\mbox{\vdots}}
\put(10,39){\vector(1,-2){14}}
\put(39,39){\vector(-1,-2){14}}
\put(25,10){\vector(0,1){9}}

\end{picture}\\ }
     \end{minipage}
\caption{strongly protected in  ${\mathcal P}_{t+1}$}  \label{spro1}
\end{figure}

If $x\to u$ is protected in ${\cal P}_{t+1}$ like   case $(1)$ in Figure \ref{spro1},  $w\to x\to u$ occurs and $w$ and $u$ are not adjacent in ${\cal P}_{t+1}$. If $w\to x$ is undirected in $e_{t+1} $, from Lemma \ref{redecs}, there exists a directed path from $y$ to $x$. Any parent of $x$ that is in this path must not be a   parent of $y$; otherwise,   there exists a directed cycle from $y$ to $y$ in $e_{t}$. Hence we have that the parent sets  of $y$ and $x$ are not equal.   This is a contradiction of the condition $\Pi_x=\Pi_y$ in Lemma \refchickervalid\, in the paper \cite{heaos}. We have that $w\to x$    and  $ x\to u$  occur in $e_{t+1}$.

If $x\to u$ is protected in ${\cal P}_{t+1}$ by v-structure $x\to u \leftarrow w$,  like case $(2)$  in Figure \ref{spro1}, clearly, the v-structure also   occurs   in  $e_{t+1} $, so $ x\to u$  occurs in $e_{t+1}$.

If $x\to u$ is protected in  ${\mathcal P}_1$ like  case $(3)$ in Figure \ref{spro1}, we have that the v-structure $w\to u\leftarrow w_1$ also occurs in $e_{t+1} $. If either $x-u$ or $u\to x$ is in $e_{t+1}$, we have that $w_1\to x$ and $w\to x$ are both in $e_{t+1}$ and the v-structure $w_1\to x\leftarrow w$ occurs. Hence  we have have   that  $x\to u$  occur in $e_{t+1} $.

Now we show that if $x\to u$ is protected in  ${\mathcal P}_{t+1}$ like $(c)$ in Figure \refsprog\,, it is also protected in  $e_{t+1}$. For any $u_1$ in $x\to u_1\to u$, there are  only two cases: $u_1$ and $y$ are adjacent or nonadjacent.

When $u_1$ and $y$ are not adjacent, like (4) in Figure \ref{spro1}, there is a v-structure $u_1\to u\leftarrow y$ in ${\mathcal P}_{t+1}$. Then $u_1\to u$ occurs in $e_{t+1}$. If $x-u$ occurs in $e_{t+1}$,  by   Lemma \ref{cq}, the edge between $x$ and $u_1$ must be directed and oriented as $u_1\to x$ in $e_{t+1}$.   This is   impossible,   because there exists some extension of ${\mathcal P}_{t+1}$  that   has an edge  oriented as  $x\to u_1$. Thus,  $x\to u$ occurs in $e_{t+1} $.

When $u_1$ and $y$ are adjacent, we have that $u_1\to y$ and $u_1-y$ do  not occur in ${\cal P}_{t+1}$ since $P_x=P_y$ must hold in $e_t$ for the validity of the operator InsertU $x-y$.   Hence we have that $y\to u_1$    occurs in ${\mathcal P}_{t+1}$ and $x\to u$ is strongly protected like case (5) in   Figure \ref{spro1}. We consider two cases: $x\to u_1$ occurs or does not occur in  $e_{t+1} $.

 Assume   $x\to u_1$ occurs   in  $e_{t+1} $. If $u_1\to u$ occurs in $e_{t+1} $, clearly, $x\to u$ must occur in $e_{t+1}$ because there is a partially directed path $x\to u_1\to u$ in $e_{t+1}$. If $u_1\to u$ is undirected in $e_{t+1}$, from Lemma \ref{cq}, $x\to u$ must   occur in $e_{t+1} $.

In case (5), we have   that   $u_1$ is also a common child of $x$ and $y$, so, $x\to u_1$ will  also be strongly protected in ${\mathcal P}_{t+1}$ from the condition \textbf{iu}$_3$.
 Now, consider $x\to u_1$;   if it is protected in ${\mathcal P}_{t+1}$ like any of case (1), (2), (3), or (4),  then, by   our proof, $x\to u_1$ occurs in  $e_{t+1} $. Thus, $x\to u$ must occur in $e_{t+1} $. If $x\to u_1$ is protected in ${\mathcal P}_{t+1}$ like case (5), we can find another vertex $u_2$   that    is a common child of $x$ and $y$ like case (6). From the proof above, we know if $x\to u_2$ occurs in $e_{t+1} $, $x\to u_1$ and $x\to u$ also occur in $e_{t+1} $. Since the graph   has    finite vertices, we can find a common child of $x$ and $y$, say $u_k$, such that $x\to u_k$ is protected in  ${\mathcal P}_{t+1}$ like one of cases (1), (2), (3)
  or   (4). Thus, $x\to u_k$ occurs in $e_{t+1}$,   implying that   $x\to u_{k-1}$ occurs in ${\mathcal P}_{t+1}$,  so,     finally, $x\to u$ occurs in ${\mathcal P}_{t+1}$.

\textbf{(Only if)} From Lemma \ref{re0str}, we have that the modified graph  ${\cal P}_{t+1}$ is also the resulting completed PDAG  $e_{t+1}$. Hence, all directed edges in  $e_{t+1}$ are strongly protected in ${\cal P}_{t+1}$, so     the Algorithm \refalgai\, will  return True.

 $\Box$

\noindent {\textbf{ Proof of  (ii) of Theorem \ref{a3810theo} }}

To prove  (ii) of Theorem \ref{a3810theo},   we need following lemma.
\begin{lemma}
\label{a8theo0}
Let $e_t $ be a completed PDAG, ${\cal P}_{t+1}$ be the PDAG obtained in Algorithm \refalgaii\, with input of a valid  operator InsertD $x\to y$, and ${e}_{t+1}$ be the resulting completed PDAG extended from ${\cal P}_{t+1}$.  We have:
 \begin{enumerate}
   \item If $u$ is not a common child of $x$ and $y$, then all directed edges $y\to u$  in ${\cal P}_{t+1}$ are also in $e_{t+1}$.
   \item All directed edges $v\to y$  in ${\cal P}_{t+1}$ are also in ${e}_{t+1}$;
  \end{enumerate}
\end{lemma}

\begin{proof}

(1)

If $u$ is not a common child of $x$ and $y$, and $y\to u$ occurs in ${\cal P}_{t+1}$, we have that there is a structure like $x\to y\to u$ in  ${\cal P}_{t+1}$. Because $x
\to y$ occurs in  ${e}_{t+1} $, $y\to u$ must  be in ${e}_{t+1}$ too.

(2)

From    Algorithm \refalgaii, all directed edges $v\to y$   are    strongly  protected in ${\cal P}_{t+1}$. When $v$ is not adjacent to $x$ in  ${e}_{t+1}$, $v\to y \leftarrow x$ is a v-structure, so $v\to y$ occurs in ${e}_{t+1}$. When  $v$ is   adjacent to $x$,    we show below that  if $v \to y$    is strongly protected like one of four cases in Figure \ref{spro11}, it is also strongly protected in $e_{t+1}$.

 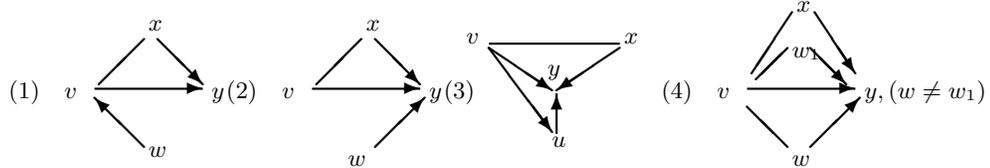
\begin{figure}[h]
 \centering
  \begin{minipage}[t]{0.22\linewidth}
   \centering
 { \unitlength=0.6mm \begin{picture}(50,40)
  \thicklines

\put(0,23){\mbox{$(1)\quad v $}}
\put(45,23){\mbox{$ y$}}
\put(31,38){\mbox{$x$}}
\put(31,10){\mbox{$ w$}}

\put(33,36){\vector(1,-1){10}}
\put(30,36){\line(-1,-1){10}}

\put(19,25){\vector(1,0){24}}
\put(30,12){\vector(-1,1){11}}
\end{picture}\\ }
     \end{minipage}
 \begin{minipage}[t]{0.22\linewidth}
   \centering
{ \unitlength=0.6mm
\begin{picture}(50,40)
  \thicklines

\put(0,23){\mbox{$(2)\quad v $}}
\put(45,23){\mbox{$ y$}}
\put(27,8){\mbox{$w$}}
\put(31,38){\mbox{$x$}}

\put(33,36){\vector(1,-1){10}}
\put(30,36){\line(-1,-1){10}}

\put(19,25){\vector(1,0){24}}
\put(33,12){\vector(1,1){11}}

\end{picture}\\ }
  \end{minipage}
    \begin{minipage}[t]{0.22\linewidth}
   \centering
 { \unitlength=0.6mm \begin{picture}(50,40)(0,5)
  \thicklines

\put(5,40){\mbox{$v $}}
\put(0,28){\mbox{$(3)$}}
\put(40,40){\mbox{$x$}}
\put(23,33){\mbox{$y$}}
\put(24,17){\mbox{$u$}}

\put(10,40){\line(1,0){29}}

\put(10,39){\vector(3,-2){14}}
\put(39,39){\vector(-3,-2){14}}

\put(10,39){\vector(3,-4){14}}
\put(25,20){\vector(0,1){9}}
\end{picture}\\ }
     \end{minipage}
    \begin{minipage}[t]{0.22\linewidth}
   \centering
{ \unitlength=0.6mm \begin{picture}(50,40)
  \thicklines
\put(0,23){\mbox{$(4)\quad v $}}
\put(45,23){\mbox{$ y, (w\neq w_1)$}}
\put(29,8){\mbox{$w$}}
\put(29,32){\mbox{$w_1$}}
\put(30,42){\mbox{$ x$}}

\put(29,42){\line(-2,-3){9}}
\put(34,42){\vector(2,-3){9}}

\put(19,25){\vector(1,0){24}}
\put(33,12){\vector(1,1){11}}
\put(28,12){\line(-1,1){9}}

\put(33,34){\vector(1,-1){8}}
\put(28,34){\line(-1,-1){7}}
\end{picture}\\ }
  \end{minipage}

\caption{strongly protected in  ${\mathcal P}_{t+1}$.  }  \label{spro11}
\end{figure}

In   case (1) of Figure \ref{spro11},  because there is no path from $y$ to $v$, we have that $w\to v$ occurs in ${e}_{t+1} $ from Lemma \ref{redecs}. Hence we have that $v\to y$ occurs in $e_{t+1}$.

In   case (2), there is a v-structure $w\to y\leftarrow v$ in ${\mathcal P}_{t+1}$.  So, $v\to y$ occurs in $e_{t+1}$.

In   case (3),  because there is no path from $y$ to $u$, we have    that    $v\to u$ occurs in $e_{t+1}$ according to  Lemma \ref{redecs}.  If $u\to y$ occurs in $e_{t+1}$,   $v\to y$    occurs in  $e_{t+1} $. If $u\to y$ become $u-y$  in $e_{t+1}$,  $v\to y$  must also be in in  $e_{t+1}$ from Lemma \ref{cq}.

From the proof of (i) of Theorem \ref{a3810theo}, we also have that  $v\to y$  must be in  $e_{t+1} $  when case (4) occurs in  ${\mathcal P}_{t+1}$.

 Notice that the above proof  also holds when we replace $x-v$ by a directed edge or add  an edge between $x$ and $w$( or $u$). Hence we have that   $v\to y$ in ${\cal P}_{t+1}$ also occurs in $e_{t+1}$.
\end{proof}

 We now give a proof  for (ii) of Theorem \ref{a3810theo}.

\textbf{(If)}

  We need to consider   four cases in Figure \refsprog\, in which  $y\to u$ is strongly protected in ${\cal P}_{t+1}$. Similar to the proof of (i) of Theorem \ref{a3810theo}, we first prove that the theorem holds in the first three cases in Figure \refsprog, which correspond to the cases $(1)'$, $(2)'$ and $(3)'$  shown in Figure \ref{spro21}. Notice that the following proof holds  for any   configuration   of   the    edge between $x$ and $w$.

 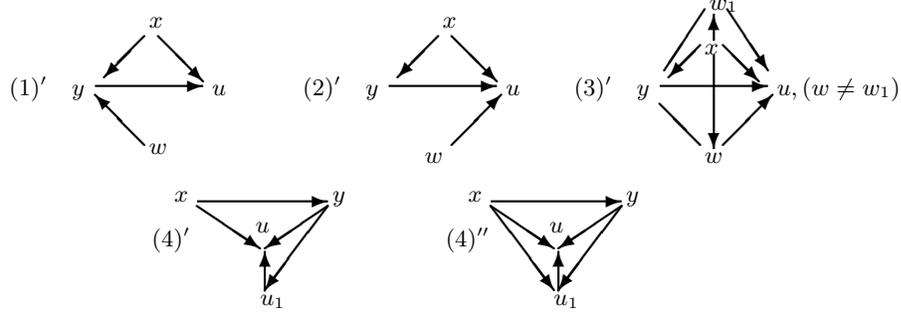
\begin{figure}[h]
 \centering

  \begin{minipage}[t]{0.3\linewidth}
   \centering
 { \unitlength=0.6mm \begin{picture}(50,40)
  \thicklines

\put(0,23){\mbox{$(1)'\quad y $}}
\put(45,23){\mbox{$ u$}}
\put(31,38){\mbox{$x$}}
\put(31,10){\mbox{$w$}}

\put(33,36){\vector(1,-1){10}}
\put(30,36){\vector(-1,-1){9}}

\put(19,25){\vector(1,0){24}}
\put(30,12){\vector(-1,1){11}}
\end{picture}\\ }
     \end{minipage}
 \begin{minipage}[t]{0.3\linewidth}
   \centering
{ \unitlength=0.6mm
\begin{picture}(50,40)
  \thicklines

\put(0,23){\mbox{$(2)'\quad y $}}
\put(45,23){\mbox{$ u$}}
\put(27,8){\mbox{$w$}}
\put(31,38){\mbox{$x$}}

\put(33,36){\vector(1,-1){10}}
\put(30,36){\vector(-1,-1){9}}

\put(19,25){\vector(1,0){24}}
\put(33,12){\vector(1,1){11}}

\end{picture}\\ }
  \end{minipage}
    \begin{minipage}[t]{0.3\linewidth}
   \centering
{ \unitlength=0.6mm \begin{picture}(60,40)(0,0)
  \thicklines
\put(0,23){\mbox{$(3)'\quad y $}}
\put(45,23){\mbox{$ u, (w\neq w_1)$}}
\put(29,8){\mbox{$ w$}}
\put(29,32){\mbox{$x$}}
\put(30,42){\mbox{$w_1$}}

\put(29,42){\line(-2,-3){9}}
\put(34,42){\vector(2,-3){9}}
\put(31,35){\vector(0,1){6}}
\put(31,32){\vector(0,-1){21}}

\put(19,25){\vector(1,0){24}}
\put(33,12){\vector(1,1){11}}
\put(28,12){\line(-1,1){9}}

\put(33,34){\vector(1,-1){8}}
\put(28,34){\vector(-1,-1){7}}
\end{picture}\\ }
  \end{minipage}

      \begin{minipage}[t]{0.3\linewidth}
   \centering
 { \unitlength=0.6mm \begin{picture}(50,40)
  \thicklines

\put(5,40){\mbox{$x $}}
\put(0,30){\mbox{$(4)'$}}
\put(40,40){\mbox{$y$}}
\put(23,33){\mbox{$u$}}
\put(24,17){\mbox{$u_1$}}

\put(10,40){\vector(1,0){29}}

\put(10,39){\vector(3,-2){14}}
\put(39,39){\vector(-3,-2){14}}

\put(39,39){\vector(-3,-4){14}}
\put(25,20){\vector(0,1){9}}
\end{picture}\\ }
     \end{minipage}
 \begin{minipage}[t]{0.3\linewidth}
   \centering
 { \unitlength=0.6mm \begin{picture}(50,40)
  \thicklines

\put(5,40){\mbox{$x $}}
\put(0,30){\mbox{$(4)''$}}
\put(40,40){\mbox{$y$}}
\put(23,33){\mbox{$u$}}
\put(24,17){\mbox{$u_1$}}

\put(10,40){\vector(1,0){29}}

\put(10,39){\vector(3,-2){14}}
\put(39,39){\vector(-3,-2){14}}

\put(10,39){\vector(3,-4){14}}
\put(39,39){\vector(-3,-4){14}}
\put(25,20){\vector(0,1){9}}
\end{picture}\\ }
     \end{minipage}
\caption{ Five cases in which  $x \to u$ or $y\to u$ is strongly protected.}  \label{spro21}
\end{figure}

Consider the case $(1)'$ in Figure \ref{spro21}. From Lemma \ref{a8theo0}, $w \to y$ occurs in $e_{t+1}$. We have that $x\to u$ must occur in $e_{t+1} $.

 Because there is  a v-structure $w\to u \leftarrow y$ in case $(2)'$, we have that   $w\to u \leftarrow y$ also occurs in $e_{t+1}$.

After implementing Algorithm \refalgaii, if case $(3)'$  occurs in  ${\mathcal P}_{t+1}$, we have that $y-w$ is not strongly protected in  ${\mathcal P}_{t+1}$ and the edge between $y$ and $w$ have opposite directions  in different consistent extensions of ${\mathcal P}_{t+1}$. Hence  $y-w$ occurs in $e_{t+1}$. Similarly,  $y-w_1$ also occurs in $e_{t+1}$. Moreover, the  v-structure $w\to u \leftarrow w_1$ occurs in $e_{t+1}$.  We have that $y\to u$ is strongly protected and occurs in $e_{t+1}$.

    We now just need to show that a directed edge $y\to u$ that is   strongly protected in ${\mathcal P}_{t+1}$ like   case    $(4)'$ ($x$ and $u_1$ are nonadjacent) or $(4)''$ ($x$ and $u_1$ are adjacent)   in Figure \ref{spro21}   is   also directed in $e_{t+1} $.

In case $(4)'$, from  delete   Lemma \ref{a8theo0}, $y\to u_1$ occurs in $e_{t+1}$. Moreover, $x\to u\leftarrow u_1$ is a v-structure, so $u_1\to u$ also occurs in ${\mathcal C}_1 $. So we have $y\to u$ must occur in $e_{t+1}$.

In case $(4)''$, we have  that     $u_1$ is also a common child of $x$ and $y$;    hence, $y\to u_1$   will   also be strongly protected in ${\mathcal P}_{t+1}$ from the condition of this Theorem.
Consider $y\to u_1$;   if it is protected in ${\mathcal P}_{t+1}$ like at least one case other than $(4)''$,  from our proof, $y\to u_1$ is also compelled in  $e_{t+1} $, so    $y\to u$ must be compelled in $e_{t+1}$. If $y\to u_1$ is protected in ${\mathcal P}_{t+1}$ like case $(4)''$, we can find another vertex $u_2$ that   is a common child of $y$ and $x$;    from the proof above, we know if $y\to u_2$ is directed in $e_{t+1}$, $y\to u_1$ and $y\to u$ are directed too. Since the graph   has    finite vertices, we can find a common child of $x$ and $y$, say $u_k$, such that $u_k$ is protected in  ${\mathcal P}_{t+1}$ like at least one case other than $(4)''$. It is compelled in $e_{t+1}$,   so   we can get $y\to u_{k-1}$ is compelled in ${\mathcal P}_{t+1}$,   so,    finally, $y\to u$ is also compelled in ${\mathcal P}_{t+1}$. We have that  $y\to u$ must occur in $e_{t+1} $ and \textbf{id}$_3$ holds.

\textbf{(Only if)} Let $u$ be a common child of $x$ and $y$ in $e_{t}$. If condition \textbf{id}$_3$ holds for a valid operator InsertD $x\to y$,  we have that $y\to u$   in $e_{t}$ occurs in $e_{t+1}$ and   is strongly protected in $e_{t+1}$.  We need to show  that   $y\to u$ must be strongly protected in ${\cal P}_{t+1}$,    obtained in Algorithm \refalgaii. From the   proof of this statement above, we know we just need to consider the five configurations in which $y\to u$ is strongly protected in ${e}_{t+1}$ in Figure \ref{spro21}.

We know that v-structures in $e_{t+1}$ occur in ${\cal P}_{t+1}$ therefore,  the v-structure in the  cases    $(2)'$, $(3)'$ and $(4)'$ in ${e}_{t+1}$ must occur in ${\cal P}_{t+1}$ too.

For case $(2)'$,  $y\to u$  is also strongly protected in ${\cal P}_{t+1}$, since the v-structure $y\to u\leftarrow w$ occurs in ${\cal P}_{t+1}$.

For case $(3)'$,  we have that (1)  the   v-structure $w_1\to u\leftarrow w$ occurs in ${\cal P}_{t+1}$; (2) $e_{t+1}$ and ${\cal P}_{t+1}$ have the same set of v-structures. Hence the  v-structure $w_1\to y\leftarrow w$ does not occur in ${\cal P}_{t+1}$. We have  that $y\to u$ is also strongly protected in ${\cal P}_{t+1}$  for any  configuration   of edges between $w_1$, $y$ and $w$.

For case $(4)'$, from Algorithm \refalgaii, $y\to u_1$ occurs in ${\cal P}_{t+1}$. Hence $y\to u$ is strongly protected in ${\cal P}_{t+1}$.

 Because the valid operator ``Insert $x\to y$" satisfies condition \textbf{id}$_3$, from Lemma \ref{re3}, we have that the operator ``Delete $x\to y$", when applied to $e_{t+1}$, results in $e_t$.    From the condition \textbf{dd}$_2$, any directed edge $v\to y$ in $e_{t+1}$ also occurs in $e_{t}$. For case $(1)'$,  we have that $v\to y\to u$ is strongly protected in ${\cal P}_{t+1}$.

Consider the case $(4)''$, we have that v-structures  $x\to u\to y$ and  $x\to u_1 \to y$ occur in $e_t$ since $e_t$ is the resulting completed PDAG of the operator ``Delete $x\to y$" from $e_{t+1}$. According to Algorithm \refalgaii, $x\to y$,  $x\to u\to y$ and  $x\to u_1 \to y$ occur in ${\cal P}_{t+1}$.   We have that $u\to u_1$ does not occur in $e_t$, otherwise $u\to u_1$ occurs in at least  one   consistent extension of ${\cal P}_{t+1}$ and consequently  $u_1\to u$ does not occur in $e_{t+1}$. To prove that $y\to u$ is strongly protected in $e_{t+1}$, we   need to show that $u_1\to u$ occurs in $e_t$.  Equivalently, we show $u_1-u$ does not occur in $e_t$. If $u_1-u$ occurs in a chain component denoted by $\tau$  in $e_t$, we have that neither $x$ nor $y$ are in $\tau$. The undirected edges adjacent to $x$ or $y$   are    in chain components different to $\tau$. Hence \textbf{id}$_3$ holds for the operator ``Insert $x\to y$", and   all parents of $\tau$ occur in $e_{t+1}$ too. We have that $u_1-u$ occurs in $e_{t+1}$ too. It's  a contradiction that $u_1\to y$ occurs in $e_{t+1}$.
  $\Box$

\noindent{\textbf{ Proof of  (iii) of Theorem \ref{a3810theo}  }}

\textbf{(If)}

 Since Algorithm \refalgaiii\, returns True, all directed edges like $v\to y$ are    strongly protected  in  ${\cal P}_{t+1}$. Consider the four configurations in which $v\to y$ is strongly protected in ${\cal P}_{t+1}$  in  Figure \ref{spro10}. Notice that ${\cal P}_{t+1}$ is obtained by deleting  $x\to y$ from completed PDAG $\mathcal{C} $,   by   Lemma \ref{fdel}, all directed edges with no vertices being descendants   of $y$ (excluding $y$)   in ${\cal P}_{t+1}$ will occur    in $e_t$.

Hence, we have the edges    $w\to v$ in case (1),   and    $v\to w$ in case (3)
will   remain  in $e_{t+1}$.  We have $v\to y$ in case (1) and case (3) must occur in $e_{t+1} $. Because   v-structures in case    (2) and case (4) will also remain   in $e_{t+1} $,    $v\to y$ in case (2) and case (4) must occur in $e_{t+1}$ too.

\begin{figure}[H]
 \centering
  \begin{minipage}[t]{0.22\linewidth}
   \centering
 { \unitlength=0.6mm \begin{picture}(50,40)
  \thicklines

\put(0,23){\mbox{$(1):v $}}
\put(45,23){\mbox{$ y$}}
\put(31,8){\mbox{$ w$}}

\put(19,25){\vector(1,0){24}}
\put(30,10){\vector(-1,1){11}}
\end{picture}\\ }
     \end{minipage}
 \begin{minipage}[t]{0.22\linewidth}
   \centering
{ \unitlength=0.6mm
\begin{picture}(50,40)
  \thicklines

\put(0,23){\mbox{$(2):v $}}
\put(45,23){\mbox{$ y$}}
\put(27,8){\mbox{$w$}}

\put(19,25){\vector(1,0){24}}
\put(33,12){\vector(1,1){11}}

\end{picture}\\ }
  \end{minipage}
    \begin{minipage}[t]{0.22\linewidth}
   \centering
{ \unitlength=0.6mm \begin{picture}(50,40)
  \thicklines

\put(0,23){\mbox{$(3):v $}}
\put(45,23){\mbox{$ y$}}
\put(29,8){\mbox{$w$}}

\put(19,25){\vector(1,0){24}}
\put(33,12){\vector(1,1){11}}
\put(17,22){\vector(1,-1){9}}
\end{picture}\\ }
  \end{minipage}
    \begin{minipage}[t]{0.30\linewidth}
   \centering
{ \unitlength=0.6mm \begin{picture}(60,40)(0,0)
  \thicklines
\put(0,23){\mbox{$(4):v $}}
\put(45,23){\mbox{$y, (w\neq w_1)$}}
\put(29,8){\mbox{$ w_1$}}
\put(28,37){\mbox{$ w$}}

\put(19,25){\vector(1,0){24}}
\put(33,12){\vector(1,1){11}}
\put(28,12){\line(-1,1){9}}
\put(33,36){\vector(1,-1){10}}
\put(28,36){\line(-1,-1){9}}
\end{picture}\\ }
  \end{minipage}

\caption{Four configurations of   $v \to y$ being strongly protected.}   \label{spro10}
\end{figure}
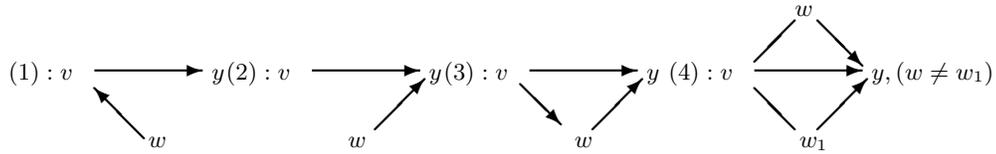

\textbf{(Only if)}
If condition \textbf{dd}$_2$ holds for a valid operator DeleteD $x\to y$, all edges like $v\to y$ ($v\neq x$) in $e_t$ will occur in $e_{t+1}$. $v\to y$ must be strongly protected in $e_{t+1}$.  Consider the four configurations in which $v\to y$ is strongly protected in ${e}_{t+1}$  as Figure \ref{spro10}. We know  that  v-structures in $e_{t+1}$ must occur in ${e}_{t}$;  consequently, all directed edges in  $e_{t+1}$ must occur in ${e}_{t}$; they  also occur in ${\cal P}_{t+1}$. From   Lemma \ref{fdel}, $w-v-w_1$ in case (4) in Figure \ref{spro10} must be in  ${\cal P}_{t+1}$, so  an edge    $v\to y$ that is strongly protected in $e_{t+1}$ is also  strongly protected in   ${\cal P}_{t+1}$.

$\Box$
\bibliographystyle{plain}
 \bibliography{AOS1125_suppl}

\end{document}